%% file: main.tex
\documentclass[nohyperref]{article}

\usepackage{microtype}
\usepackage{graphicx}
\usepackage{subfig}
\usepackage{booktabs} 
\usepackage{bbm}

\usepackage{hyperref}

\usepackage[accepted]{icml2022}

\usepackage{amsmath}
\usepackage{amssymb}
\usepackage{mathtools}
\usepackage{amsthm}

\usepackage[capitalize,noabbrev]{cleveref}

\theoremstyle{plain}
\newtheorem{theorem}{Theorem}[section]
\newtheorem{proposition}[theorem]{Proposition}

\theoremstyle{definition}
\newtheorem{definition}[theorem]{Definition}

\theoremstyle{remark}

\usepackage[textsize=tiny]{todonotes}

\icmltitlerunning{Leveraging Approximate Symbolic Models for RL via Skill Diversity}

\begin{document}

\input{commands}
\twocolumn[
\icmltitle{Leveraging Approximate Symbolic Models for Reinforcement Learning via Skill Diversity}



\icmlsetsymbol{equal}{*}

\begin{icmlauthorlist}
\icmlauthor{Lin Guan}{equal,asu}
\icmlauthor{Sarath Sreedharan}{equal,asu}
\icmlauthor{Subbarao Kambhampati}{asu}
\end{icmlauthorlist}

\icmlaffiliation{asu}{School of Computing \& AI, Arizona State University, Tempe, AZ}

\icmlcorrespondingauthor{Lin Guan}{lguan9@asu.edu}

\icmlkeywords{RL, Planning, Neuro-Symbolic AI}

\vskip 0.3in
]



\printAffiliationsAndNotice{\icmlEqualContribution} 

\input{0-abstract}
\input{1-introduction}
\input{2-relatedWork}
\input{3-problemSetting}

\input{4-algorithms}

\input{5-evaluation}

\input{6-conclusion}
\input{7-acknowledgement}

\bibliography{bib}
\bibliographystyle{icml2022}

\clearpage
\appendix
\input{Supplementary_File}

\end{document}

%% file: commands.tex
\newcommand{\KGRLFull}{Approximate Symbolic-Model Guided Reinforcement Learning}
\newcommand{\KGRL}{ASGRL }
\newcommand{\KGRLAlgo}{ASGRL}

\newcommand{\MVTRFull}{minimally viable task representation}
\newcommand{\MVTR}{MVTR}

\newcommand{\Shortcite}[1] {\citeauthor{#1}~\citeyear{#1}}

%% file: 0-abstract.tex
\begin{abstract}
    Creating reinforcement learning (RL) agents that are capable of accepting and leveraging task-specific knowledge from humans has been long identified as a possible strategy for developing scalable approaches for solving long-horizon problems. While previous works have looked at the possibility of using symbolic models along with RL approaches, they tend to assume that the high-level action models are executable at low level and the fluents can exclusively characterize all desirable MDP states. Symbolic models of real world tasks are however often incomplete. To this end, we introduce {\em Approximate Symbolic-Model Guided Reinforcement Learning}, wherein we will formalize the relationship between the symbolic model and the underlying MDP that will allow us to characterize the incompleteness of the symbolic model. We will use these models to extract high-level landmarks that will be used to decompose the task. At the low level, we learn a set of diverse policies for each possible task subgoal identified by the landmark, which are then stitched together. We evaluate our system by testing on three different benchmark domains and show how even with incomplete symbolic model information, our approach is able to discover the task structure and efficiently guide the RL agent towards the goal.
\end{abstract}

%% file: 1-introduction.tex
\section{Introduction}
In recent years, reinforcement learning (RL) methods have demonstrated an impressive ability in tackling many hard sequential-decision making problems. However, most practical reinforcement learning (RL) systems still struggle in solving long-horizon tasks with sparse rewards. Part of the challenge comes from the fact that traditionally RL methods tend to focus on agents that start \emph{tabula rasa} and acquire task information purely from experience (i.e., experience ironically sampled from expert specified simulators). While theoretically, one could inject knowledge about the task through careful reward engineering, such methods are generally non-intuitive and hard for non-AI experts to specify and may result in unanticipated side effects \cite{Menell17Invers}. This has resulted in interest in imbuing RL systems the ability to accept advice from humans through more intuitive means. Most of the earlier works in this direction focused primarily on accepting advice about agent objectives \cite{icarte18aUsing}, though recent works have looked at developing systems that accept task-level advice. In particular, symbolic planning models \cite{geffner2013concise} have been considered as a viable method for the specification of task information \cite{lyu2019sdrl, IllanesYIM20Symbolic}. Unfortunately, most of these works tend to have strong requirements (sometimes implicit ones) on the correctness of the symbolic models provided.

\begin{figure}
\begin{center}
\centerline{\includegraphics[width=0.32\columnwidth]{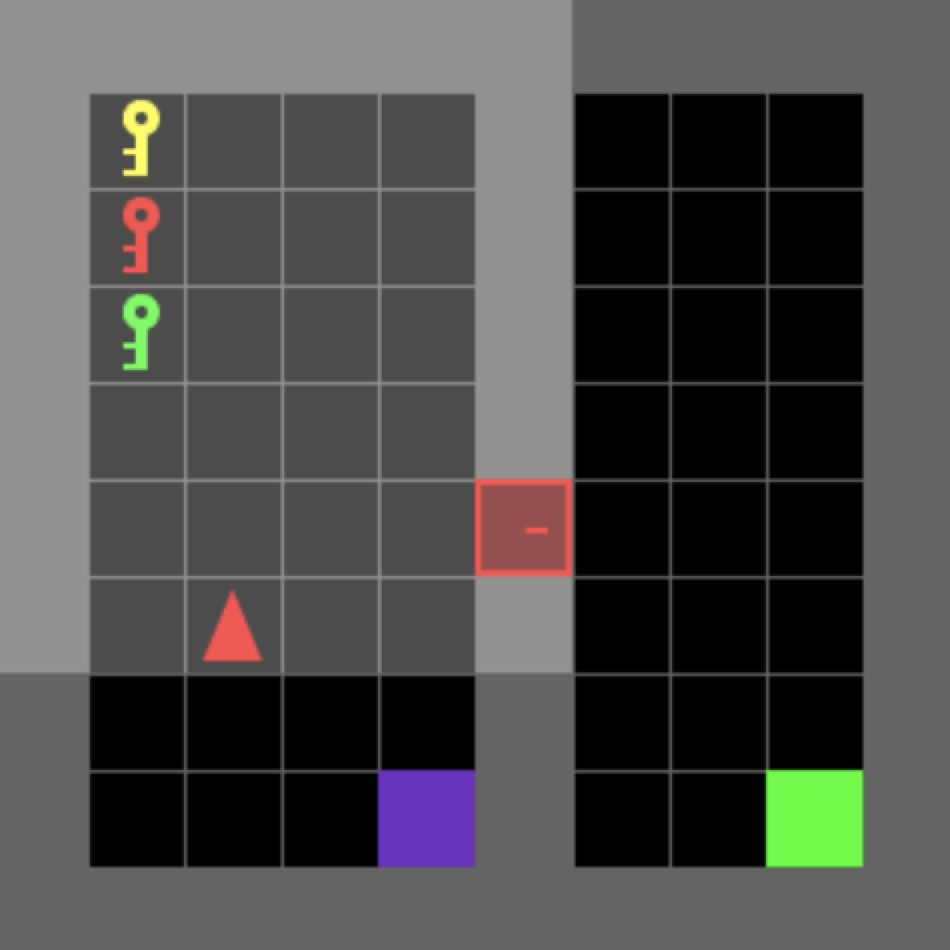}}
\caption{The Household environment. To reach the final destination (green block), the robot (red triangle) has to pick up the red key, charge itself at the purple block, and open the red door.}
\label{fig:dock-env}
\end{center}
\vspace{-0.8cm}
\end{figure}

To illustrate the importance of accounting for possible incompleteness in symbolic models, consider a simple household robotics domain (henceforth referred to as the Household environment) where the task is for a robot to visit a particular location represented by the green block (Fig. \ref{fig:dock-env}). The robot can do that by picking up the red key, then re-charging, then opening the door to visit the final location. A human user could potentially help the robot in its learning process by providing various pieces of information related to the task. For example, providing information such as the location of the destination and the fact that the door is locked and would require a key to open. However, such information, potentially provided as a symbolic model, need not be a complete representation of the task and may even contain incorrect information. For one, they may have forgotten to mention the fact that there are multiple keys in the house, and that only one of them can open the door. They may thus have incorrectly specified that the robot can use any of the keys. In this case, the symbolic model only \textbf{partially specifies} the prerequisites for opening the door. Secondly, the user may not be a robotics expert and might not know that this particular robot model has limited battery capacity and would require recharging itself in the middle of the task (by visiting the charging dock). In this case, features related to the charging dock and the robot's battery level may be \textbf{completely missing} in the symbolic model. Thirdly, the user would expect the door to remain ajar after the robot enters the room, but in reality the door will be automatically closed once it enters the room. In this case, the symbolic model might \textbf{incorrectly specify} that the effect of the action of passing through the door is both that the robot is in the destination room and the door is still ajar. As a result, existing approaches \cite{IllanesYIM20Symbolic,lyu2019sdrl} that expect a correct and complete model will fail due to multiple reasons. For one, the robot will never learn a policy that will allow it to achieve a state where the door is still open and the robot is in the destination room. If the robot tries to myopically learn a policy to pick up a key, it will only pick up the nearest key (which happens to be the wrong key); and finally even if it picks up the right key, learning to unlock the door (as per the symbolic model) will just leave the robot's battery drained as it doesn't know that it needs to charge itself before attempting to unlock the door.

In this work, we hope to address this challenge, by proposing a framework called \KGRLFull~(\KGRLAlgo) \footnote{Our source code is available at \url{https://github.com/GuanSuns/ASGRL}.} that will allow RL agents to leverage approximate symbolic models, i.e., models that may be incomplete and may contain incorrect information. In this framework, we will formalize the relationship between the symbolic model and the true task being solved by the agent, which will allow the model to be approximate while still encoding useful information. We will show how given such a relationship, we can extract task decomposition information in the form of landmarks from the model that will act as subgoals for the overall task. We will then use these subgoals in the RL problem by learning a set of diverse low-level policies aimed at achieving them. 


We note that the main motivation for our work is to allow for human advice to take the form of approximate symbolic models of the task--even if those are incorrect and incomplete. Such incorrectness and incompleteness in symbolic models may either be the result of lack of full knowledge on the part of the humans, or because the advice was extracted indirectly--either from specialized text descriptions of the domain (c.f. \cite{DBLP:conf/ijcai/FengZK18, gpt3-to-plan}) or from general large language models (e.g. \cite{saycan2022arxiv}). 
We will show that the way our RL system uses this symbolic advice preserves completeness under fairly permissive conditions (characterized by the MVTR conditions defined in Section~\ref{sec:mvtr}). In addition to potential efficiency and sample complexity improvements, such advice can also improve interpretability as it aligns with human expectations. It's also worth noting that, the advice provided by the symbolic model may include human preferences and safety constraints, apart from domain information.


%% file: 2-relatedWork.tex
\section{Related Work}
\label{sec:rel-works}

There has been increasing interest in incorporating human knowledge into reinforcement learning systems for task specification or better sample efficiency \cite{zhang2019leveraging}. In particular, works have argued for developing methods to incorporate human guidance in the form of symbols, which is a natural way for humans to express knowledge \cite{zhang2018composable, kambhampati2022symbols}. One particular form is task knowledge encoded as symbolic planning models \cite{geffner2013concise}. Some prominent works in this direction include \cite{yang2018peorl,IllanesYIM20Symbolic,lyu2019sdrl,kokel2021reprel}. Most of these works assume that the high-level plans are in some way an executable entity. They assume that the mapping from the high-level actions to potentially temporally extended operators are either given or can to be learned. 


However, human understanding of a task is incomplete by nature and human-defined symbols are known to be imprecise, so some works both in symbolic planning \cite{myers-advisable} and reinforcement learning have allowed the user to actively participate in policy learning by continually providing feedback to refine the reward function \cite{basu2018learning, guan2021widening}. In this work, we do not consider additional repeated human supervision, as they can be costly and even when available can be applied on top of the method discussed here. 

Some works in symbolic planning also allow for inaccurate models, but they do not consider the use of reinforcement learning \cite{rao-model-lite, nguyen2017robust, DBLP:conf/rss/VemulaOBL20, vemula2021cmax++}.
Outside the use of full symbolic models, people have also considered the use of other forms of high-level advice like policy sketches \cite{andreas2017modular}, natural language instructions \cite{goyal2019using} and temporal logic specifications \cite{icarte18aUsing,de2019foundations,jothimurugan2021compositional}.
Such advice tends to be a lot more restrictive than the information that can be encoded in planning models. 

As we will see later, we incorporate a Quality-Diversity objective in policy learning. Diversity of plans and policies has been studied both in symbolic planning \cite{tuan-diverse} and reinforcement learning for better robustness \cite{haarnoja2017reinforcement, Kumar2020OneSI} and better exploration \cite{florensa2017stochastic, Achiam2018VariationalOD, Eysenbach2019DiversityIA, lee2019efficient}. In this work, we show that diversity can also be used to make up for the incompleteness in symbolic knowledge.  



%% file: 3-problemSetting.tex
\section{Background}
\label{sec:background}
This paper focuses on the basic problem studied in standard RL settings, namely one with an agent acting in an unknown environment, trying to learn a policy that can maximize its expected value. Let the agent task correspond to an infinite horizon discounted MDP $\mathcal{M}=\langle S,A,R,T,s_0,\gamma\rangle$, where $S$ is the set of (possibly infinite number of) states, $A$ the set of actions, $R: S \rightarrow \mathbb{R}$ a possibly sparse reward function and $T: S\times A\times S\rightarrow [0,1]$ be the transition function and $s_0\in S$ is the starting state for the agent. In particular, we will focus on cases where the task is goal directed such that, there exist a set of goal states $S_{G} \subseteq S$. These are absorbing states in the sense that, for $s\in S_{G}$, $T(s,a,s')=0$ for all actions $a$ and state $s'\neq s$. The reward function takes the form of a sparse function that assigns $0$ reward to all states except the goal states, i.e, $\forall s \in S_{G}, R(s) = \mathcal{R}^{G}$, where $\mathcal{R}^{G}$ will be referred to as the goal reward. We will refer to this unknown MDP $\mathcal{M}$ as the task MDP. The reward function or goal states may be part of the task that the agent is interacting with or may be specified at the beginning of the learning phase by the user of the system.

In this setting, the solution takes the form of a deterministic, stationary policy that maps states to actions. A value of a policy $\pi$, denoted as $V^\pi: S \rightarrow \mathbb{R}$ provides the expected cumulative discounted reward obtained by following the policy from a given state. A policy is said to be optimal if there exists no policy with a value higher than the current policy. Finally an execution trace of a policy from a state $s$ corresponds to a state-action sequence with non-zero probability that terminates at a goal state (denoted as $\tau \sim \pi(s)$, where $\tau = \langle s,\pi(s),...,s_k\rangle$). In this particular case, since the reward is only provided by the goal, the value of the policy in a state is directly proportional to the probability of reaching the goal state under the given policy. We will capture this by the notation $P_\mathcal{G}(s|\pi)$, which is given as $$P_\mathcal{G}(s|\pi) = \Sigma_{\tau \sim \pi(S)} P( \tau|\pi),$$
where $\tau$ are possible traces ending in a goal state and $P( \tau|\pi)$ is its likelihood under the given policy $\pi$. 

In the paper, we consider cases where the user or a domain expert provides task-specific knowledge captured in the form of a declarative action-centered representation of a planning task. In particular, we will focus on STRIPS-like planning models \cite{Fikes71STRIPS}, where a planning model is defined in the form $\mathcal{P}=\langle F,A,I,G\rangle$. $F$ is a set of propositional state variables or fluents that defines the symbolic state space (i.e., $S^{\mathcal{P}}= 2^F$, i.e., each symbolic state corresponds to a set of binary fluents that are true).
The possible fluents for the household environment could include facts like \texttt{has-key} (a fluent capturing whether the robot has a key in its position), \texttt{door-open} (whether the door is open), etc. 
In our case, the fluent set will be the only component that we will assume to be grounded directly to the true model. We will in fact assume access to a function $\mathcal{F}: S\times F \rightarrow \{0,1\}$ mapping MDP states to symbolic fluents, such that $\mathcal{F}(s,f)$ is interpreted as fluent $f$ being true in MDP state $s \in S$. In practice, this mapping function can be constructed using learned binary classifiers \cite{Sreedharan20Bridge, zhang2018composable, lyu2019sdrl}. Every other mapping from the symbolic model components to the MDP and vice-versa will be defined over $\mathcal{F}$ or will be extracted while making use of $\mathcal{F}$.

Furthermore, $A$ is the set of action definitions, where each action $a\in A$ is defined further as $a=\langle \textit{prec}^a,\textit{add}^a,\textit{del}^a,\rangle$. Here $\textit{prec}^a$ is a set of binary features that are referred to as precondition and an action is said to be executable in a state only if the precondition feature are true in that state. $\textit{add}^a \subseteq F$ and $\textit{del}^a\subseteq F$ capture effects of executing the action, where  $\textit{add}^a$ called add effects capture the set of feature set true by executing the action and $\textit{add}^a$ called delete effects capture the set of feature set false by executing the action. Thus the effect of executing in a symbolic state $s^\mathcal{P}$ can be denoted as

\vspace*{-0.2in}
\[a(s^\mathcal{P})=\begin{cases}(s^\mathcal{P}\setminus\textit{del}^a)\cup \textit{add}^a, & \textrm{if}~\textit{prec}^a\subseteq s^\mathcal{P}\\
\textit{undefined}, &\textrm{otherwise}
\end{cases}\]
\vspace*{-0.1in}

Finally, $I \in S^\mathcal{P}$ is the initial state of the agent and $G \subseteq F$ is the goal specification. 
Without loss of generality, we will assume that the goal specification always consists of a single goal fluent. 
In the Household environment, the initial state is given as $I=\{\texttt{at-starting-room}\}$ and the goal is defined as  $G=\{\texttt{at-destination}\}$.
The approximate symbolic model used for the Household environment can be found in the Appendix \ref{appendix:model}.
All states that satisfy the goal specification (i.e., the goal fluents are true in the given state) are considered a valid goal state, and the set of all goal states are given as $S^{\mathcal{P}}_G = \{s^\mathcal{P} \mid~s^\mathcal{P}\in S^\mathcal{P} \wedge G \subseteq s^\mathcal{P}\}$. Note that the above definition captures a set of actions with deterministic effects (i.e., there is no uncertainty associated with the execution of an action), thus the solution for a deterministic planning problem is a plan. A plan consists of a sequence of actions which when executed in the initial state, result in a goal state, i.e., $\pi^{\mathcal{P}} = \langle a_1^{\mathcal{P}},...,a_n^{\mathcal{P}}\rangle$ is considered a valid plan, if
\[\pi^{\mathcal{P}}(I) = a_n^{\mathcal{P}}(...(a_1^{\mathcal{P}}(I))....) \in S^{\mathcal{P}}_G.\]
The symbolic state sequence corresponding to a plan is the sequence of symbolic states obtained by applying the actions in the plan to the initial state, i.e., for the plan $\pi^{\mathcal{P}}$, we have a state sequence of the
form $\tau_{\pi^{\mathcal{P}}}=\langle s_0^{\mathcal{P}},...,s_k^{\mathcal{P}} \rangle$, such that $s_0^{\mathcal{P}} = I$ and 
$s_i^{\mathcal{P}} = a_i^{\mathcal{P}}(...(a_1(s_0^{\mathcal{P}})...)$. With the definition of a plan and a symbolic state sequence in place, we are now ready to define the actual formal problem tackled by our approach.

%% file: 4-algorithms.tex
\section{\KGRLFull}
\label{sec:method}
In this section, we will introduce the first formal framework to quantify the relationship between true task models and STRIPS models that may contain imprecise and incomplete domain knowledge (Section \ref{sec:mvtr}). Then, based on the formal framework, we will present our solution strategy that can extract useful subgoal information from the inaccurate symbolic model and construct reward functions that guide the RL agent towards the final goal state (Section \ref{sec:approach}).

\subsection{Minimally Viable Task Representation}
\label{sec:mvtr}
An \KGRL agent aims to find a policy to reach a goal state starting from an initial state $s_0$ in an (unknown) MDP $\mathcal{M}$. The \KGRL problem setting expects the agent to have access to a declarative model of the form $\mathcal{P}^\mathcal{M}$. This model is meant to capture a high-level representation of the task, possibly given by a domain expert. Such models are particularly well suited for task knowledge specification, since in addition to being quite general, these have their origins in folk psychological concepts and are as such easier for people to understand and specify \cite{miller2019explanation,chakraborti2019plan}. We also chose to use a deterministic model to encode the high-level information since people are generally known to be poor probabilistic reasoners \cite{tversky1993probabilistic}. Although, when available these models could easily be extended to allow for such information.

To proceed, we first define the concept of a trace being an instantiation of a symbolic plan as follows  
\begin{definition}
Given a symbolic plan for $\mathcal{P}$ of the form $\langle a_1^\mathcal{P},...,a_k^\mathcal{P}\rangle$ with a corresponding symbolic state sequence $\tau^\mathcal{P}=\langle I,..,s_{i}^{\mathcal{P}},..,s_k^{\mathcal{P}}\rangle$ (where $s^\mathcal{P}_i = a_i^\mathcal{P}(...(a_1^\mathcal{P}(I))..)$), a trace $\tau=\langle s_0,....,s_k\rangle$ (i.e., MDP state sequence) is said to be an \textbf{instantiation of the symbolic plan},
if the relative ordering of the fluents established by the symbolic state sequence is reflected in the trace, i.e., for any ordered fluent pair $f_i \prec f_j$ according to $\tau^\mathcal{P}$, there must exist a state $s_n$ in $\tau$ such that $\mathcal{F}(s_n,f_j)=1$ (and it is not true in any earlier states), and there exists a state $s_m \in \tau$ such that $m < n$ and $\mathcal{F}(s_m,f_i)=1$.
\end{definition}

In addition, throughout this paper, we will be using the given high-level symbolic model to extract relative ordering between fluents, which is defined as follows
\begin{definition}
A symbolic state sequence $\tau_{\pi^{\mathcal{P}}}=\langle s_0^{\mathcal{P}},...,s_k^{\mathcal{P}} \rangle$  is said to establish a relative ordering $f_i \prec f_j$ (to be read as fluent $f_i$ precedes $f_j$), if $s_n^{\mathcal{P}}$ is the first state where $f_j$ is true and there exists a state $s_m^{\mathcal{P}}$ such that $m < n$ where $f_i$ is true.
\end{definition}
The relative ordering information captures the fact that according to the model at hand, it is possible to establish the fact $f_i$ before achieving the fact $f_j$.

Now, to establish the relationship between the underlying task and the high-level task information provided by the domain expert, we will introduce the concept of \emph{Minimally Viable Task Representation} (MVTR). MVTR will allow us to capture the fact that the model provided to us by the domain expert, by the very nature of it being an approximation of the task, is going to be incomplete and may even contain incorrect information about the task. But at the same time, the model could still contain useful information that can be leveraged by the agent to achieve its objectives. 

Let us note a few ways the setting is more permissive: (a) We should only expect that the symbolic model is a minimal characterization of the task in that it might only partially capture information about a \emph{single} trace that leads to the goal in some policy. In other words, the model may allow for several other plans that may contain invalid information. (b) We do not expect any symbolic state to completely capture the information in low-level states. As we will see later, one symbolic state may correspond to a diverse set of low-level states with different utilities for the given task. (c) We do not even expect that any of the symbolic plans at the high levels are executable in any meaningful sense (i.e., each action in the plan can be mapped to an exact temporally extended operator that can be executed at low level). For us, actions are merely a conceptual tool employed by the model specified to encode their knowledge about the \emph{order in which the various facts need to be established}.

All of this stands in stark contrast to many of the previous methods that require a much more \emph{complete} symbolic model and in many cases, the decision-making problem turns merely into a search for the exact plan that can be executed in the task MDP. We believe that our method is a lot more realistic, as it corresponds to the case where the model-specifier may have some specific strategies in mind for the agent to achieve the goal, but may have overlooked many of the contingencies and possible side-effects. Additionally, the model-specifier may not even be an expert on all aspects of the problem and as such may be oblivious to certain state variables.

More formally, we introduce MVTR, which characterizes a \emph{more relaxed condition} under which there is still a guarantee that we can extract usable task-relevant information from an imperfect symbolic model $\mathcal{P}$.
\begin{definition}
\label{def:min_viable}
For an MDP $\mathcal{M} = \langle S,A,R,T,s_0,\gamma\rangle$ with a goal state set $S_{\mathcal{G}}$, a symbolic planning model $\mathcal{P} = \langle F,A,I,G\rangle$ is said to be a \textbf{minimally viable task representation} if the following two conditions are met
\begin{enumerate}
    \item \textit{The fluents that are part of goal specification are only true in goal states, or there exists a subset of goal states $\hat{S}_{\mathcal{G}} \subseteq S_{\mathcal{G}}$, such that (a) $\forall s\in \hat{S}_{\mathcal{G}}$ and $ \forall f_g\in G$,  $\mathcal{F}(s,f_g)=1$; (b) for any state $\hat{s}$ in $(S \setminus \hat{S}_{\mathcal{G}})$, there exists a fact $f_g\in G$, such that $\mathcal{F}(\hat{s},f_g)=0$}.
    \item \textit{
    There exists at least one goal-reaching trace $\tau=\langle s_0,....,s_g\rangle$, where $s_g \in S_{\mathcal{G}}$, which is an instantiation of a symbolic plan for $\mathcal{P}$;
    in other words, there exists a policy $\pi$ that leads to the goal with a non-zero probability from the initial state $s_0$ in the MDP $\mathcal{M}$ (i.e., $P_{\mathcal{G}}(s_0|\pi) > 0$) such that at least one goal-reaching trace $\tau$ sampled from $\pi$ will satisfy the relative fluent ordering encoded in the symbolic plan for $\mathcal{P}$
    }.
\end{enumerate}
\end{definition}

The relaxation of the requirement of individual symbolic actions to correspond to specific temporally extended operators mirrors the intuition that has been established in multiple previous works regarding the semantics of abstract actions. Namely, providing an exact and concise definitions of abstract or temporally extended actions is quite hard \cite{SrivastavaRP16Meta,MarthiRW07Angelic}. While we do not leverage the full semantics of angelic non-determinism employed by some of the earlier work, the MVTR does allow for the fact that the effects of an action may be a loose characterization of a set of reachable states and there may not be one exact state that satisfies all action effects. For example, in the Household environment, the human might mistakenly specify the effects of the action \texttt{pass\_through\_door} as \texttt{door-ajar} and \texttt{at-final-room}, though in reality, \texttt{door-ajar} only holds when the robot is passing through the door, and once the robot is in the final room, the door will become closed. Here, without expecting \texttt{pass\_through\_door} to be executable, an MVTR will only capture the fact that \texttt{door-ajar} and \texttt{at-final-room} are useful characteristics for the underlying task. Also, note that MVTR's definition of the relative ordering does not place any ordering between fluents that are achieved at the same step in the symbolic model (e.g., \texttt{door-ajar} and \texttt{at-final-room}).

Appendix \ref{appendix:general-mvtr} includes additional discussion on the generality of the MVTR requirement, including the fact that it captures cases where the symbolic model is a state abstraction of the task MDP. Also note that while most of the above discussion focuses on cases where the symbolic model is given as guidance for an RL agent trying to solve a task, like in the case \cite{IllanesYIM20Symbolic}, the symbolic model may itself be used as a way for the user to specify the RL agent task (or to define task reward). In such cases, the goal states $S_\mathcal{G}$ are completely specified by the features the user includes in the symbolic goal specification.

We will now see how we can derive information that can be used by an RL agent once we are given a symbolic model that is known to be an MVTR for the true task.
We should note that like {\em subgoal serializability} \cite{Korf-serializability,kambhampati1996candidate}, MVTR is a {\em normative condition} in that when it happens to hold, the completeness of RL using the advice is guaranteed. Our focus in this paper is not on providing necessary and sufficient syntactic conditions on the symbolic model to guarantee MVTR (although, as we argued, MVTR is a fairly permissive condition in as much as it only requires at least one final trajectory to be consistent with the symbolic model). 

\subsection{Subgoal Extraction and Skill Learning}
\label{sec:approach}
Once we have a minimally viable symbolic model in place, the first question we need to answer is what information from the symbolic model can we use to guide the RL agent to the goal state(s)? First off, since the action effects may not be directly achievable in a state, we can't learn temporally extended operators for the task MDP. One could try to use the relative orderings that are encoded in the plans, but as Definition \ref{def:min_viable} puts it, only some of the plans capture true ordering information, and even when they do, the information is just ordering information. Thus iterating over all possible plans and testing whether they encode useful information could be a hard learning problem. Instead, as we will see, we can extract a single type of information i.e., fact landmarks and their relative orderings, that is guaranteed to hold in the task MDP.
\begin{definition}
For a given planning problem $\mathcal{P}=\langle F, A, I, G\rangle$, the fact landmarks are given by the tuple $\mathcal{L}=\langle \hat{F}, \prec_{\mathcal{L}}\rangle$ such that $\hat{F} \subseteq F$ and $\prec_{\mathcal{L}}$ defines a partial ordering between elements of $\hat{F}$, such that if for $f_1, f_2 \in \hat{F}$, we have $f_1 \prec_{\mathcal{L}} f_2$, then the relative ordering $f_1 \prec f_2$ is satisfied by the symbolic state sequence corresponding to every valid plan in $\mathcal{P}$.
\end{definition}
That is, landmarks encode information that is valid in all plans and thus is also valid in the one(s) that capture information of the goal reaching policy, which leads us to
\begin{proposition}
The relative ordering established by the landmarks corresponding to a minimally viable symbolic model should hold in at least one trace that can be sampled from a policy for the task MDP with a non-zero probability of reaching the goal state.
\end{proposition}

We can extract these landmarks through efficient algorithms like the one discussed by \Shortcite{KeyderRH10Sound}. In particular, we will focus on facts that appear in action preconditions (this will further avoid unnecessary side effects). Note that, goal facts are always landmark facts and there will be a precedence ordering between all other facts and the goal facts. Some of the landmark facts for the Household environment would be $\texttt{has-key}\prec\texttt{door-open}\prec\texttt{at-final-room}$.

Now, these landmarks provide a natural way to decompose the full task into subgoals and we can even provide the relative ordering between the subgoals. These landmarks thus allow us to use a hierarchical reinforcement learning framework \cite{kulkarni2016hierarchical}, where we can first learn how to achieve these subgoals using temporally extended operators and then learn a meta controller that will try to use these operators to achieve the eventual goal. In particular, we will learn options \cite{sutton1999between} for each subgoal.

\begin{definition}
A subgoal skill for a given landmark fact $f$ is an option for the task MDP $\mathcal{M}$ of the form $\mathcal{O}_f = \langle \mathbb{I}_f, \mathbb{G}_f, \pi_f\rangle$, where $\mathbb{I}_f \subseteq S$ is the initiation set of the option, $\mathbb{G}_f \subseteq S$ is the termination set of the option, and finally $\pi_f$ is the policy corresponding to the option, such that: (a) $\forall s \in  \mathbb{G}_f$, we have $\mathcal{F}(s,f)=1$; (b) if there exists no landmark fact $f' \prec f$, then $ \mathbb{I}_f = \{s_0\}$, otherwise $\forall s \in  \mathbb{I}_f$, there exists a landmark fact $f' \prec f$ such that $\mathcal{F}(s,f')=1$.
\end{definition}

Thus these subgoal skills are meant to drive the system from states that satisfy some previous landmark fact to the next one. However, extracting landmarks and their relative orderings from an MVTR only relaxes the requirement for precise symbolic action models. We still need to address the challenges arising from the fact that there may be many low-level states that satisfy a given landmark fact, but they may not all be equivalent in terms of how easy it is to reach the goal from those states. For example, if one were to learn a skill for the landmark \texttt{door-open}, the simplest policy to learn would be the one where the robot goes to open the door directly (assuming the key pick up has already been completed). Now the skill would be able to successfully open the door, but once the robot has opened the door it would be out of battery and will not be able to perform any other actions. Thus it requires us to not only achieve the subgoal \texttt{door-open} but also do so with a high battery level (thus requiring visiting the charging point before reaching the door). We can't identify this from the high-level symbolic information alone, since it contains no information about the battery level. Instead, in this work, we will try to make up for this lack of information by learning diverse skills that visit diverse low-level states.\footnote{For those familiar with the automated planning literature, the approach of finding diverse set of options is similar, in the case of completely known models, to serializing subgoals by considering a variety of subplans--including non-minimal ones for each subgoal; see \cite{kambhampati1996candidate}}

\noindent\textbf{Learning Diverse Skills Per Subgoal. }  Our approach involves learning a set of options for each landmark fact with a diverse set of termination states with the given fluent true, i.e, for a given landmark $f$, we try to learn a set of skills $\mathbb{O}_f$ (where $|\mathbb{O}_f|$ is set to a predefined count $k$) such that for each skill $o_f^z \in \mathbb{O}_f$, we have a set of \emph{skill terminal states} $\mathbb{G}_f^z$ such that $\forall s \in \mathbb{G}_f^z, \mathcal{F}(s,f)=1$. This means we treat the states that satisfy the landmark $f$ as absorbing subgoal state(s) and end current skill training as soon as it enters such a state. The individual skills are learned in the order specified by the landmark set $\mathcal{L}$ and the subgoal states obtained as part of learning immediately preceding skills are used as initial states for the succeeding skill.

Recall that as fluents do not uniquely define all states that are equally useful in the current task, we will need to learn to reach a diverse set of landmark-satisfying states to ensure we can find the truly useful one(s). Here, we employ an information-theoretic objective to encourage diversity in skill terminal states while still ensuring that the landmark is satisfied. In particular, for a landmark fact $f$, let us use the random variable $Z_f$ to represent the specific skill being followed (where $Z_f$ takes values from $z^1_f$ to $z^k_f$)
and let $G_f$ denote the random variable corresponding to being in one of the possible terminal state(s) that are discovered by all the skills (i.e., $G_f$ can take values from the set $\bigcup_{i=1}^{k} \mathbb{G}_f^i$). Then our objective is to learn a set of $k$ skill policies that achieve the landmark fact $f$, while minimizing the conditional entropy:
\small
\begin{equation}
\label{eq:entropy-objective}
    min~\mathcal{H}(Z_f|G_f)
\end{equation}
\normalsize
The conditional entropy is minimized for landmark-achieving skills when the policies are reaching distinct skill terminal states. Intuitively, the objective is optimized when we can easily infer the index of skill $z^i_f$ by only looking at the low-level landmark state.

The problem of learning the policy for a skill $z^i_f$ for each landmark $f$ is framed as a separate RL problem with a new reward function given as
\small
\begin{equation} 
\label{eq:diverse-reward}
    R_f(s) = 
\begin{dcases}
    \mathcal{R}_{\mathcal{L}} + \alpha_H * R_{d}(s|z^i_f) & \text{if }\mathcal{F}(s,f)=1 \\
    0              & \text{otherwise}
\end{dcases}
\end{equation}
\normalsize
Where, $R_{\mathcal{L}}$ is the reward associated with achieving any landmark state and is usually set to 1, and the diversity reward $R_{d}$ can be computed as: 
\small
\begin{equation} 
\label{eq:diverse-log-prob}
R_{d}(s) = log(p(z^i_f|s)).
\end{equation}
\normalsize
$\alpha_H$ is a hyper-parameter such that $-1 < \alpha_H * R_{d} \leq 0$ (where $R_{d}$ is clipped). This reward design resembles the rewards in some previous Quality-Diversity policy learning approaches \cite{florensa2017stochastic, Eysenbach2019DiversityIA, Achiam2018VariationalOD}. However, ours differs from theirs in the fact that we are using the diversity objectives for a different purpose and our diversity reward is assigned only at skill terminal states. 

This formulation immediately results in two theoretical guarantees: (a) The achievement of subgoal is always prioritized in our reward function. (b) The system will prefer achieving subgoal states that are not visited by other skills.

The formal propositions and proofs can be found in Section \ref{appendix:prop-goal-reachability} and Section \ref{appendix:prop-prefer-unvisit} in Appendix. The two propositions state the fact that our diversity-augmented reward always encourages the skills to cover all reachable subgoal states with distinct state coverage. In the case of the Household environment, the landmark \texttt{door-open} will correspond to two possible reachable states, one where the robot has \texttt{door-open}, while holding the right key and has no charge left and the other one where the door is open, the robot is holding the right key and is charged. So if we have $k \geq 2$, we will have at least one skill where the robot has opened the door with charge.

\noindent \textbf{Calculating Diversity Rewards. } When $z^i_f$ is sampled uniformly at the beginning of each episode, $p(z^i_f|s)$ for a skill terminal state $s$ can be estimated by counting the state visitations as in \cite{florensa2017stochastic}:
\small
\begin{equation} 
\label{eq:discrete-prob-z}
p(z^i_f|s) \simeq \frac{count(s, z^i_f)}{\sum_{z^j_f \in Z_f}{count(s, z^j_f)}}.
\end{equation}
\normalsize

A practical problem we may encounter is how to appropriately set the number of diverse skills to be learned (i.e., the hyper-parameter $k$), which is usually unknown upfront. Here, we propose to gradually increase the value of $k$ until a specified maximum is reached or no new skill terminal state is discovered. However, under this curriculum setting, we can not use Eq. \ref{eq:discrete-prob-z} to estimate $p(z^i_f|s)$ because $z^i_f$ is no longer sampled according to a uniform distribution. Alternatively, we can apply the Bayes theorem and replace the absolute count value with the estimated probability of reaching any skill terminal state $s$ after executing skill $z^i_f$:
\small
\begin{equation} 
\label{eq:state-dist-bayes}
\begin{split}
p(z^i_f|s) & \simeq \frac{p(z^i_f,s) }{p(s)} \\
& = \frac{p(s|z^i_f)p(z^i_f) + \alpha_L}{\sum_{z^j_f \in Z_f}{p(s|z^j_f)p(z^j_f)} + |Z_f| * \alpha_L},
\end{split}
\end{equation}
\normalsize
where $\alpha_L$ is the Laplace smoothing factor, and the prior $p(z)$ can be computed from the total number of rollouts and the total number of $z^i_f$ being executed, and $p(s|z^i_f)$ can be estimated by sampling $N$ traces and counting the terminal state visitations:
\small
\begin{equation} 
\label{eq:estimate-terminal-dist}
p(s|z^i_f) \simeq \frac{count(s, z^i_f)}{\sum_{s' \in \mathbb{G}_f}{count(s', z^i_f)}}
\end{equation}
\normalsize

Additional discussion on how to compute the diversity rewards in continuous domains can be found in Appendix \ref{appendix:clustering}.\looseness=-1

\noindent \textbf{Training the Meta Controller. }  Since multiple skills are learned, we will need a meta-controller that learns to select the skill  $o_f^z$ to execute for each landmark subgoal and to achieve the original task reward. In particular, our meta-controller RL problem consists of state space $S_{meta}$. The action space consists of the diverse skill set corresponding to the landmarks $\mathcal{L}$ (represented as $\mathbb{O}_\mathcal{L}$). The reward function $\mathcal{R}_{meta}$ is a sparse binary reward that has 0 on all states except the ones that satisfy the final goal facts. Numerically, $\mathcal{R}_{meta}$ is identical to a binary success-indicating environment reward. 

Note that our meta controller is not restricted to any specific type of meta-state representation. Here we list two possible options: (a) we can define $S_{meta}$ as the sequence of executed skills (henceforth referred to as the history-based representation); (b) we can use low-level MDP states as $S_{meta}$ (i.e., initial states and landmark-satisfying states). Additional discussion on why the history-based representation is a sufficient and more compact representation can be found in Appendix \ref{appendix:meta-state}. Our meta-controller follows standard Q-Learning, and it is trained together with low-level skills. Algorithm \ref{algo:meta-controller} in Appendix provides the pseudo code for our learning method. The algorithm starts by sampling a possible linearization for the given set of landmarks. Then for each landmark fact in the sequence, one of the diverse skills is selected by the meta-controller according to $\epsilon$-greedy. We start with uniform skill selection ($\epsilon_0=1$) and slowly anneal the exploration probability by a fixed factor after a pre-specified training step.

%% file: 5-evaluation.tex
\section{Evaluation}
\label{sec:evaluation}

\begin{figure}
    \centering
    \subfloat[\centering MineCraft]{{\includegraphics[width=0.24\textwidth]{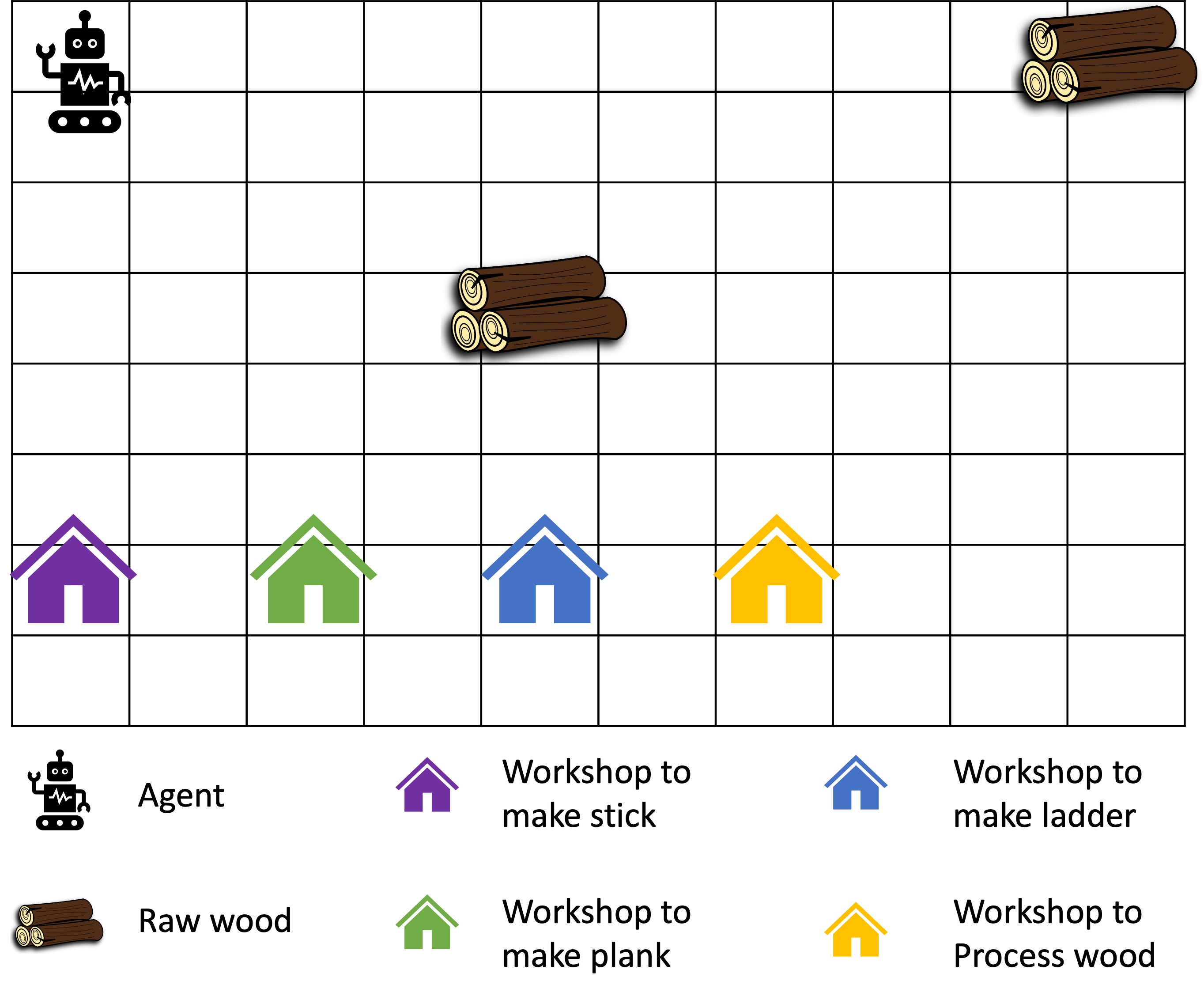}} \label{fig:minecraft-domain}}
    \subfloat[\centering  Mario]{{\includegraphics[width=0.24\textwidth]{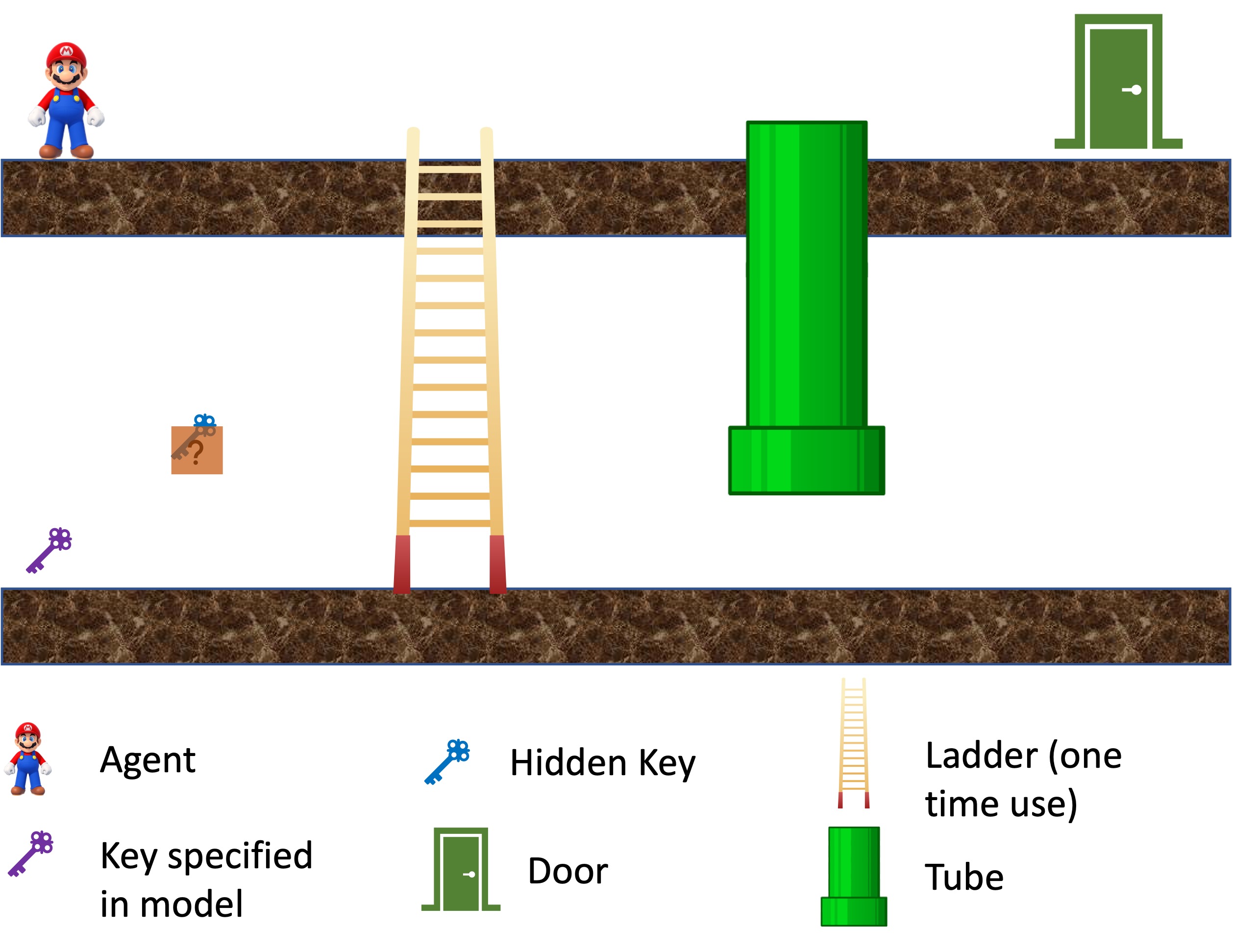}} \label{fig:mario-domain}}
    \caption{Visualization of the MineCraft environment and the Mario environment.}
    \label{fig:domain-images}
\end{figure}

\begin{figure*}[t]
    \centering
    \subfloat[\centering Average evaluation success rates ]{{\includegraphics[width=0.25\textwidth]{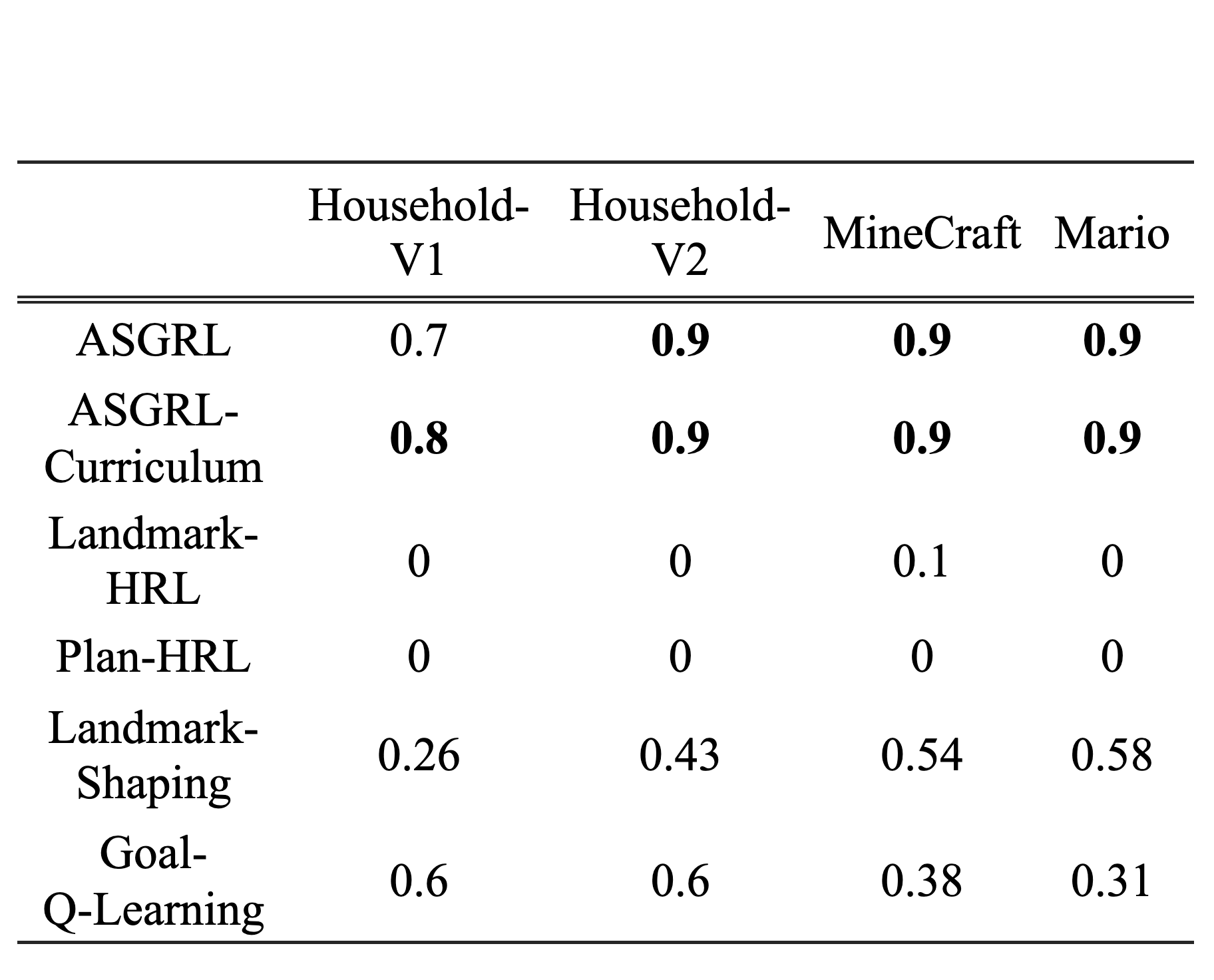}} \label{fig:success-rate}}
    \subfloat[\centering Smoothed learning curves]{{\includegraphics[width=0.7\textwidth]{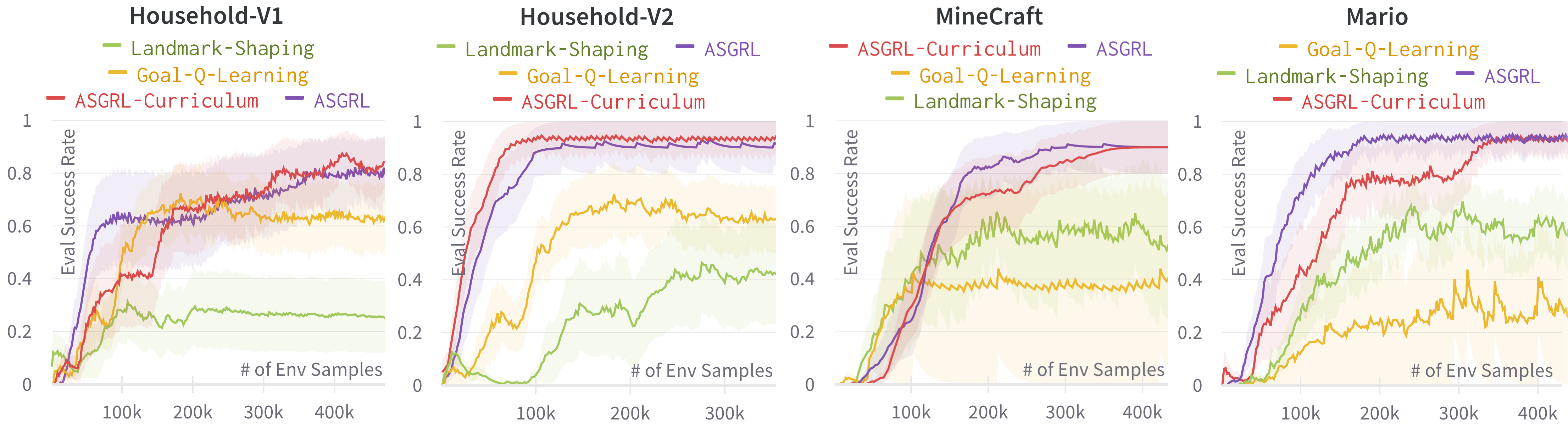}}\label{fig:results}}
    \caption{Comparisons of \KGRL with other baselines. The solid lines in the right figure show the mean score over 10 random seeds. The shaded regions represent the standard error of the mean.}
    \label{fig:rate-results}
\end{figure*}

We evaluate the performance of our approach in three environments: the Household environment, a MineCraft environment, and a Mario game environment. These three environments all require long-horizon sparse-reward task learning. We aim to answer the following questions in evaluation: firstly, whether our approach can extract useful task information from approximate MVTR domain models; secondly, whether the learning agents can leverage the extracted information and the diversity objective to efficiently find goal-reaching policies.

The performance metrics we consider include the success rate of solving the given task and the sample efficiency. The final success rate was obtained by running the learned policies with a low exploration probability 10 times and counting the cases that the agent succeeds to reach the final goal state(s) within certain environment steps. To track the sample efficiency, we evaluated the agents every 5 training episodes, during which the meta-controller and each skill policy act greedily. Each algorithm was run 10 times with different random seeds and the average results are reported.

\subsection{Baselines and Implementations}

We compare our approach to the following baselines:

\noindent \textbf{Plan-HRL:} this baseline follows the implementation of existing symbolic plan guided RL approaches like TaskRL \cite{IllanesYIM20Symbolic} and PEORL \cite{yang2018peorl}. Plan-HRL learns separate RL policy for each symbolic operator and uses hierarchical reinforcement learning to reach the goal by following plans from the (incomplete) symbolic models.

\noindent \textbf{Landmark-HRL:} this baseline uses the same meta-controller and low-level RL agents as ours. However, it doesn't have the diversity objective (that is, it only uses $R_{\mathcal{L}}$ as rewards) and it only learns one single policy for each landmark subgoal. 

\noindent \textbf{Landmark-Shaping:} this baseline follows the idea of plan-based reward shaping \cite{grzes2008plan}. Rather than following any plan from error-prone symbolic models,  Landmark-Shaping uses potential-based shaping rewards as a heuristic to softly guide the RL agent towards the final goal. The state potential is given as the number of landmark fluents that have been satisfied by the agent. 

\noindent \textbf{Goal-Q-Learning:} this baseline uses standard Q-Learning to learn from a sparse binary final-goal-reaching reward. 

Recall that there could be two versions of our approach, namely the one with standard learning setting (denoted as \textbf{\KGRLAlgo}) and the one with curriculum setting (denoted as \textbf{\KGRLAlgo-Curriculum}). For evaluation purposes, we set a large enough $k$ for each skill to cover all possible landmark states, though it's more practical to employ the curriculum setting all the time for any unknown skill. In the experiments, the history-based meta-state representation is used in \KGRL. The implementation details of our approach and the baselines can be found in Appendix \ref{appendix:implementation-details}.

\subsection{Environments and Results}
\noindent \textbf{The Household environment (Fig. \ref{fig:dock-env}). } 
We consider two possible versions of incomplete symbolic models for the Household domain. In the first version, the human expert knows about the existence of the charging dock, but he/she is unaware of the fact that there are multiple different keys in the house and believed that the robot can open the door with any key. We denote the learning task using knowledge from this version of symbolic model as Household-V1. Hence, Landmark-HRL fails as it only learns to pick up the nearest key. Also, Plan-HRL never succeeds as the action $\texttt{pass\_through\_door}$ is not executable.

In other cases, the human may not even know the need for recharging and the fact that the door is locked. So the human might provide a smaller symbolic model with only action $\texttt{go\_to\_destination}$ and $\texttt{pass\_through\_door}$ being described. We denote this learning task as Household-V2. In this case, Landmark-HRL fails because it only learns to open the door without bothering to recharge the robot. Plan-HRL also fails since the operator $\texttt{pass\_through\_door}$ is not executable.

The approximate symbolic models and extracted landmarks are presented in Appendix \ref{appendix:model-dock-v1} and \ref{appendix:model-dock-v2}.

\noindent \textbf{The MineCraft environment (Fig. \ref{fig:minecraft-domain}). } This is a variation of the environment used in \cite{andreas2017modular}. Here the agent can navigate around the environment and collect raw materials to build tools. We consider the task of making ladder from plank and stick. To accomplish this task, the agent needs to collect raw wood, bring it to workshop 1 to get processed wood, and then make a plank and a stick, which can then be used to build a ladder at different workshops. Our version of the task differs from the original one in the size of the ladder to build. In this case, the agent has to collect multiple pieces of wood. However, the human expert doesn't know this additional requirement and gives an inaccurate symbolic model (see Appendix \ref{appendix:model-minecraft} for the approximate model and extracted landmarks) in which the actions $\texttt{make\_plank}$ and $\texttt{make\_stick}$ only require the agent to bring one piece of wood to workshop 1 (i.e., $\texttt{wood-processed}$ is True). 

In this case, Plan-HRL and Landmark-HRL learn a policy for the operation $\texttt{get\_processed\_wood}$ that myopically collects only one piece of raw wood and heads to workshop 1 in a shortest path. In contrast, the diversity objective in our approach gives extra incentive to the agent for visiting the wood processing workshop with different numbers of raw pieces of wood. 

\noindent \textbf{The Mario environment (Fig. \ref{fig:mario-domain}). } This environment is a modified version of the well-known Atari game Montezuma's Revenge. The task for the Mario agent is to open the door by going downstairs, picking up the two keys (one hidden in the red rock), and going upstairs. This Mario environment differs from Montezuma's Revenge in the following three aspects: firstly, to open the door, the agent needs to pick up both keys; secondly, the ladder here is already worn out, so it will break after being used once; thirdly, Mario can only go down through the tube, not up. Hence, the optimal plan is to go down through the tube and go up through the ladder. However, an expert in Montezuma's Revenge might be unaware of all the facts above and give an inaccurate symbolic model that (a) has no fluents associated with the tube, (b) has no fluents associated with the hidden key, (c) has no fluents indicating whether the ladder is broken, and (d) contains action $\texttt{go\_up\_the\_ladder}$ that is not executable after the ladder is broken. The approximate symbolic model can be found in Appendix \ref{appendix:model-mario}.

Our approach is able to successfully learn diverse skills that go downstairs by the ladder or the tube, and skills that pick up one or both keys. In contrast, Plan-HRL fails because it only tries to go down by the ladder, and Landmark-HRL fails because it learns the easiest way to go downstairs (by the ladder as it is closer) and never takes extra effort to pick up the hidden key. Again, Goal-Q-Learning and Landmark-Shaping fail due to the difficulty of this long-horizon task. 

\noindent \textbf{Summary of main experiments. } The results are summarized in Fig. \ref{fig:rate-results}. We can see that even with incomplete and inaccurate symbolic models our approach can still solve all the tasks with high success rates, while all other baselines fail. This highlights the importance of accommodating for incompleteness (captured via MVTRs) in the symbolic model and incorporating the diversity objective in low-level skill learning. The failure of Goal-Q-Learning and Landmark-Shaping further confirms the necessity of developing approaches like \KGRL that could conserve and leverage task-hierarchy knowledge from incomplete symbolic models.

\subsection{Additional Experiments}
We conducted three additional experiments to get more insights into \KGRL (Appendix \ref{appendix:result-complementary}): (a) In Appendix \ref{appendix:clustering}, we present a simple clustering-based approach to adapt \KGRL to continuous domains. We evaluate it in a Pixel-Mario environment, in which images are used as MDP states. Results suggest that the clustering-based approach can effectively scale \KGRL to problems with large continuous state spaces. (b) We investigated whether \KGRL can work efficiently when an accurate and ``complete" symbolic model is given. Results show that \KGRL is a general approach that can be applied to both incomplete models and ``complete" models. (c) We also investigated whether different meta-state representations yield significantly different performance. Results confirm that \KGRL is not restricted to any particular meta-state design.

%% file: 6-conclusion.tex
\section{Conclusion}
\label{sec:discussion}
Giving high-level human advice to RL systems is an effective way of scaling them and aligning them to human preferences. In this paper, we present \KGRLFull~ an RL framework capable of leveraging incorrect and incomplete symbolic models to solve long-horizon sparse reward goal-directed tasks. We saw how landmarks provide robust task-decomposition information under minimal assumptions and how diversity at low-level skills can help make up for missing information at the symbolic level. Our experiments show the effectiveness of our method on several domains. 

Going forward there are several exciting directions for the work. One would be to further relax the MVTR condition, in such cases, landmarks are still helpful information, but one may no longer be able to provide any of the guarantees afforded by the current assumption. Currently, we allow any low-level state features to contribute to the diversity objective, so learning diverse skills can be intractable or expensive when the state space is extremely large and there may be multiple trivial low-level state features. In the future, we would like to focus on how one could further reduce the hypothesis space of useful low-level landmark states. 


%% file: 7-acknowledgement.tex
\subsection*{Acknowledgments}
This research is supported in part by ONR grants N00014- 16-1-2892, N00014-18-1- 2442, N00014-18-1-2840, N00014-9-1-2119, AFOSR grant FA9550-18-1-0067, DARPA SAIL-ON grant W911NF19-2-0006 and a JP Morgan AI Faculty Research grant.

%% file: Supplementary_File.tex
\section{Generality of MVTR}
\label{appendix:general-mvtr}

To see the generality of the MVTR condition, let us look at a very commonly considered mapping to symbolic models, namely one based on state aggregation based abstractions.

\begin{proposition}
For a given MDP $\mathcal{M}=\langle S,A,R,T,s_0\rangle$ with goal state set $S_\mathcal{G}$, let $F$ be a set of binary features that defines a set of symbolic state $S^F = 2^F$. Consider a surjective mapping $\phi$ between $S$ and $S^F$, such that there exists a symbolic state set $S^F_G \subseteq S^F$, such that $\not\exists s \in (S\setminus S_\mathcal{G}), \phi(s)\in S^F_G$. Let $\mathcal{M}^\phi$ be an abstraction of  $\mathcal{M}$ defined using $\phi$ per \Shortcite{LiWL06Towards}. Then the determinization of $\mathcal{M}^\phi$ is a minimally viable symbolic model for $\mathcal{M}$.
\end{proposition}
This comes from the fact that the state abstraction conserves all traces and one could create an all-outcome determinization which again creates a deterministic model that conserves the model. Therefore the all outcome determinization of the abstract model will be a model that meets MVTR condition. It shows that a direct goal conserving abstraction already results in minimally viable symbolic models without placing additional requirements like conserving optimal Q-values/optimal policy action, or even that aggregated states share the same immediate reward and place restrictions on transitions possible in the abstract model \cite{LiWL06Towards}, requirements that are sometimes expected by other works leveraging symbolic abstractions of the task (c.f. \cite{kokel2021reprel}). Moreover, an MVTR doesn't even require the symbolic models to be valid abstractions of the underlying task in that, the symbolic model may contain features or actions that do not correspond to any low-level state or transition possible in the underlying MDP.

\begin{algorithm}[tb]
   \caption{Training Meta-Controller}
   \label{algo:meta-controller}
\begin{algorithmic}
   \STATE {\bfseries Input:} $\mathcal{L}$ (landmarks), $k$ (the number of diverse skills to learn for each landmark)
   \STATE Initialize $Q$-values of the meta-controller
   \REPEAT
       \STATE $\tau_{\mathcal{L}} \sim \mathcal{L}$ 
       \STATE $s_{meta}=\langle\rangle$
       \FOR{each landmark $f$ in $\tau_{\mathcal{L}}$}
            \STATE Select skill $o_f^i$ by using $\epsilon$-greedy on $Q(s_{meta},\mathcal{O}_f)$.
            \STATE Sample a trajectory $\tau$ with $o_f^i$, that either results in a landmark-satisfying state for $f$ or fails to achieve the sub-goal in max time step.
           
            \IF{finished with failure}
                \STATE Q($s_{meta}$, $o_f^i$) $\leftarrow$ $0$
                \STATE Update low-level skill policy for $o_f^i$ with a sparse reward of 0. 
                \STATE Terminate and restart from initial state(s).
            \ELSE
                \STATE Q($s_{meta}$, $o_f^i$) $\leftarrow$ $R_{meta}(s_{meta}$, $o_f^i)$ + $max_{f',j}$ Q([$s_{meta}$, $o_f^i$], $o_{f'}^j$).
                \STATE Update $o_f^i$'s state visitation counts.
                \STATE Update low-level policy for $o_f^i$ with a sparse reward of $1 + \alpha * R_{d}$.
            \ENDIF
            \STATE $s_{meta}\leftarrow$ [$s_{meta}$, $o_f^i$].
        \ENDFOR
    \UNTIL{Learning completes}
\end{algorithmic}
\end{algorithm}

\section{Proposition \ref{prop:goal_reachability} and the Proof}
\label{appendix:prop-goal-reachability}
\begin{proposition}
\label{prop:goal_reachability}
Given the reward defined in Eq. \ref{eq:diverse-reward} and sufficient exploration, each skill $o_f^z$ is guaranteed to learn a policy that has a non-zero probability of reaching a skill terminal state from some of the states in the initiation set.
\end{proposition}

\begin{proof}[Proof Sketch]
Recall that when an MVTR is given, there must exist a trace starting from some state in the initiation set that can end in a skill terminal state. Let $S$ denote the entire state space, $\mathbb{G}_f^{*}$ denote the space of skill terminal states for landmark $f$, we can show that any trace $\tau_0$ that ends in a skill terminal state will always have a greater discounted cumulative reward than any trace $\tau_0$ that never visits a terminal state:
\small
\begin{equation}
\begin{split}
V(\tau_0) & = \lambda^{^{T_0}}R_f(s_{T_0} \in \mathbb{G}_f^{*}) + \sum_{t=0}^{T_0-1} \lambda^t R_f(s_t \in S \setminus \mathbb{G}_f^{*}) \\
& = \lambda^{^{T_0}}R_f(s_{T_0} \in \mathbb{G}_f^{*}) \\
& > 0\\
& = V(\tau_1) = \sum_{t=0}^{T_1} \lambda^t R_f(s \in S \setminus \mathbb{G}_f^{*})
\end{split}
\end{equation}
\normalsize
\end{proof}

\section{Proposition \ref{prop:prefer-unvisit} and the Proof}
\label{appendix:prop-prefer-unvisit}
\begin{proposition}
\label{prop:prefer-unvisit}
Let $\mathbb{O}_f$, be a set of learned skills such that, there exist two skills $o_f^i$ and $o_f^j$, that share some reachable termination states ($\mathbb{G}^i_f\cap\mathbb{G}^j_f\neq\emptyset$). Consider a new skill set $\hat{\mathbb{O}}_f$, which is formed by only replacing $o_f^i$ with a new skill $\hat{o}_f^i$ such that $\hat{o}_f^i$ only reaches goals that are not achieved by skills in $\{\mathbb{O}_f-o_f^i\}$, then it is guaranteed that (a) values of the policies of the $i^{th}$ skill and the $j^{th}$ skill will be greater in $\hat{\mathbb{O}}_f$ and (b) the values for all the other skills will be equal or greater in $\hat{\mathbb{O}}_f$.
\end{proposition}

\begin{proof}[Proof Sketch]
Recall that the sparse reward function we use contains two components, namely a goal-reaching component and a diversity component. In this case the value function for a skill can be decomposed into two components, i.e., $V^\pi(s) = V^\pi(s)_{R_{\mathcal{L}}} + V^\pi(s)_{R_d}$, where $V^\pi(s)_{R}$ is the value component associated with $R_{\mathcal{L}}$ and $V^\pi(s)_{R_d}$ be the value component associated with $R_d$.

According to proposition \ref{prop:goal_reachability}, when the skills are optimized, they always learn policies that reach at least one terminal state. Therefore, the value component corresponding to achieving the landmark is identical between skills in $\mathbb{O}_f$ and skills in $\hat{\mathbb{O}}_f$, i.e., $V^{\pi^i_f}_{R_\mathcal{L}} = V^{\hat{\pi}^i_f}_{R_\mathcal{L}}$. Hence, the factor affecting the values of the skills is $R_d$.

The proof of part (a) directly follows from the fact there exists an $s \in \mathbb{G}^i_f\cap\mathbb{G}^j_f\neq\emptyset$ such that $P(z_j|s) < 1$. After the $i^{th}$ skill is replaced, the likelihood of state $s$ being reached by a skill is distributed across the other skills that have $s$ in it's terminal state, so for any skill $o_f^j$ with $P(z_j|s) > 0$, we have $p(\hat{z_j}|s) > p(z_j|s)$ under the new skill set $\hat{\mathbb{O}}_f$. As the $V^{\pi^j_f}_{R_\mathcal{L}}$ remains unchanged, $V^{\pi^j_f}$ must be larger. 

Similarly, since the $i^{th}$ skill is now visiting previously unvisited states $\hat{S}$ ($\forall s \in \hat{S}, P(\hat{z_i}|s)=1$) and has the same value for $V^{\hat{\pi}^i_f}_{R_\mathcal{L}}$, we are guaranteed that $V^{\hat{\pi}^i_f}$ becomes larger. 

For part (b), the introduction of the new $\hat{o}_f^i$ can only reduce the number of states that are visited by multiple skills, hence the value either increases or stays the same.
\end{proof}

\section{Calculating Diversity Rewards in Continuous State Space} 
\label{appendix:clustering}
When the state space is continuous or high dimensional, it will be intractable to directly compute Eq.\ref{eq:discrete-prob-z} or Eq.\ref{eq:state-dist-bayes}. Previous works have addressed this challenge by fitting a regressor (e.g., that is parameterized by deep neural networks) by maximum likelihood estimation \cite{florensa2017stochastic, Achiam2018VariationalOD, Kumar2020OneSI, Eysenbach2019DiversityIA}. However, this is not a feasible solution in our case, because training a regressor requires a balanced set of training samples in the form of $(s, z_f)$ pairs. Due to the randomization in RL and the fact that some skill terminal states require less exploration to be reached, our system tends to first learn skills that go to some easier-to-reach states, which will result in an imbalanced dataset during the skill learning process. To this end, we use K-Means clustering algorithm to map continuous states into a discrete space represented by the cluster index. 

Since the clustering parameters should be dynamically updated as the agent explores to unseen landmark states, we use a buffer to store $M$ recently visited states, so that whenever the clustering parameters are updated, we can relabel all the data in the buffer and use them to update the state visitation counters. The number of target clusters is set to the number of diverse skills to be learned. By default, we update the clustering parameters whenever a new added state has a greater distance to the assigned centroid than any other state in the same cluster. Note that clustering algorithms like K-Means rely on a good distance metric to work well. In our implementation, we simply use visual distance (i.e., the Euclidean distance in pixel space) as the metric, although we agree that other advanced distance metrics could lead to some improvement. Doing this allows us to use the same approach as in the discrete state space.

\section{History-based Meta-State Representation}
\label{appendix:meta-state}
The history-based state representation is meant to capture the fact that the meta controller is only trying to chain the skills until it can apply a skill for achieving the goal fluent (Proposition \ref{prop:meta1}). This is generally a more compact representation compared to representation consisting of low-level landmark-satisfying states (which may contain an infinite number of states when the space is continuous), particularly if repetition of already executed skills is disallowed (which means execution history is bounded by $|\mathbb{O}_\mathcal{L}|$). 

As we will see in Proposition \ref{prop:meta1}, a history-based representation is said to be sufficient, if we can reproduce the goal-reaching traces captured by the original MVTR model, by following the skill sequence in the history after the necessary skills have been learned. 
\begin{proposition}
\label{prop:meta1}
Let $\tau$ be a trace that is captured by the MVTR model and let $S_{\mathcal{L}} = \langle s_1^\mathcal{L},..,s_k^\mathcal{L}\rangle$ be the ordered sequence of low-level landmark states from $\tau$ capturing landmark facts and satisfies their respective ordering. Then a history-based representation of the meta controller state is sufficient if for every consecutive pair of states $s_i^\mathcal{L}, s_{i+1}^\mathcal{L}$, there exists a skill with a non-zero probability of achieving $s_{i+1}^\mathcal{L}$ from $s_i^\mathcal{L}$.
\end{proposition}

\begin{proposition}
\label{prop:meta2}
It is sufficient to keep track of the last executed state, if the optimal value for all states in the termination state of a learned skill are equal.
\end{proposition}

\section{Results of Additional Experiments}
\label{appendix:result-complementary}

\subsection{Performance in Continuous Domains}
\label{appendix:result-continuous}
We conducted an additional experiment in the Mario environment, in which images are used as MDP states. Our approach (in both standard learning setting and curriculum setting) achieves an average 0.9 (out of 1.0) success rate in the Pixel-Mario environment, which is comparable to that in the discrete version. The learning curves are shown in Fig. \ref{fig:result-mario-image}. This result confirms that our clustering-based approach is able to scale \KGRL to problems with large continuous state space. 

\subsection{Performance when Complete and Correct Symbolic Models Are Given}
\label{appendix:result-accurate-model}
We investigated whether \KGRL can also work well when an accurate symbolic model is given (in which all symbolic actions are executable and landmark fluents can uniquely capture MDP states with non-zero goal-reaching probability). With accurate symbolic models, our approach, Plan-HRL and Landmark-HRL can all efficiently solve the tasks with a 1.0 success rate. The learning curves are shown in Fig. \ref{fig:result-accurate}. This confirms that \KGRL is a more general approach that can be applied to both incomplete domain models and ``complete" models. 

\subsection{Performance when Different Meta-State Representations Are Used}
\label{appendix:different-meta-s}
The version of \KGRL that uses low-level MDP states as $S_{meta}$ is denoted as \KGRLAlgo-Meta-MDP (or \KGRLAlgo-Curriculum-Meta-MDP if the curriculum setting is used). Results (Fig. \ref{fig:result-diff-meta}) suggest that there is no significant performance difference.

\begin{figure}[t]
\centering
\includegraphics[width=0.2\textwidth]{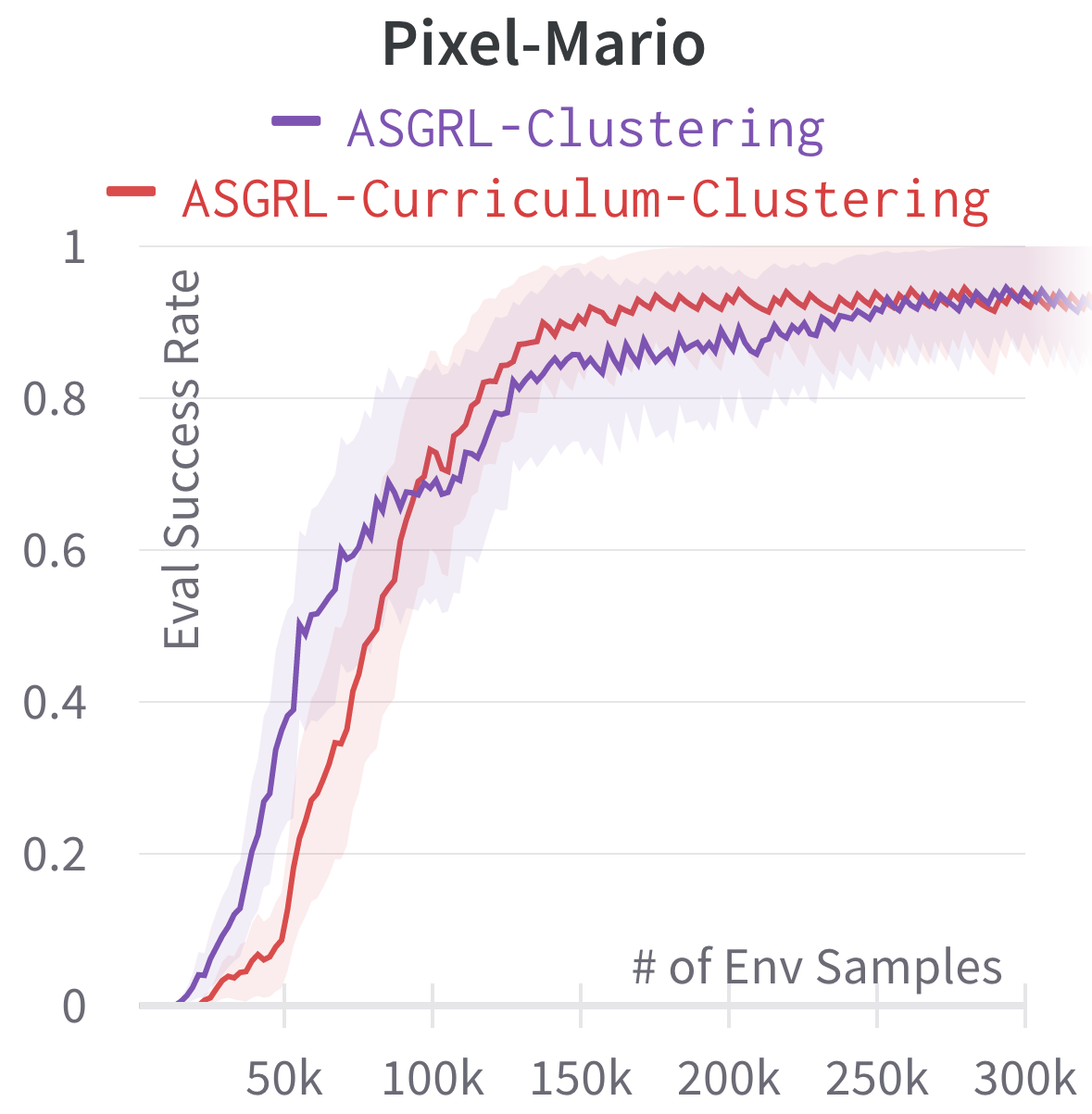}
\caption{The learning curves of our approach in Pixel-Mario.}
\label{fig:result-mario-image}
\end{figure}

\begin{figure}[t]
\centering
\includegraphics[width=0.45\textwidth]{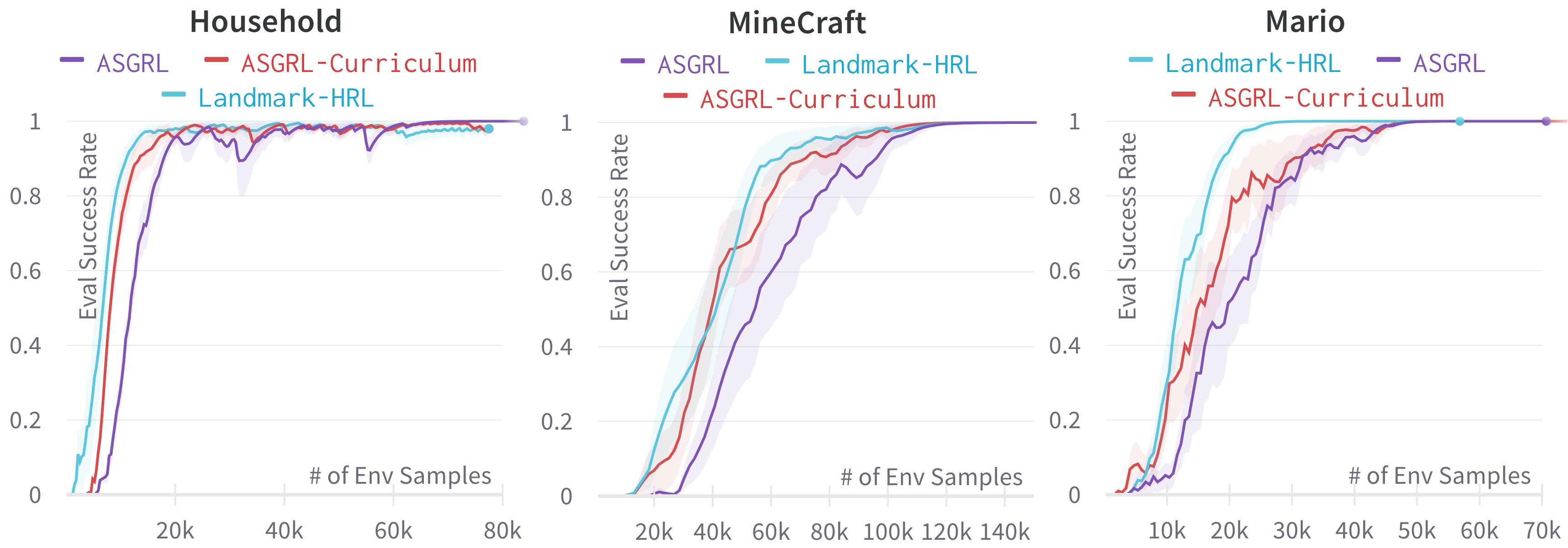}
\caption{The learning curves of our approach and Landmark-HRL when complete and correct symbolic models are provided. The curve of Plan-HRL is omitted in this figure because it almost overlaps with the curve of Landmark-HRL.}
\label{fig:result-accurate}
\end{figure}

\begin{figure}[t]
\centering
\includegraphics[width=0.45\textwidth]{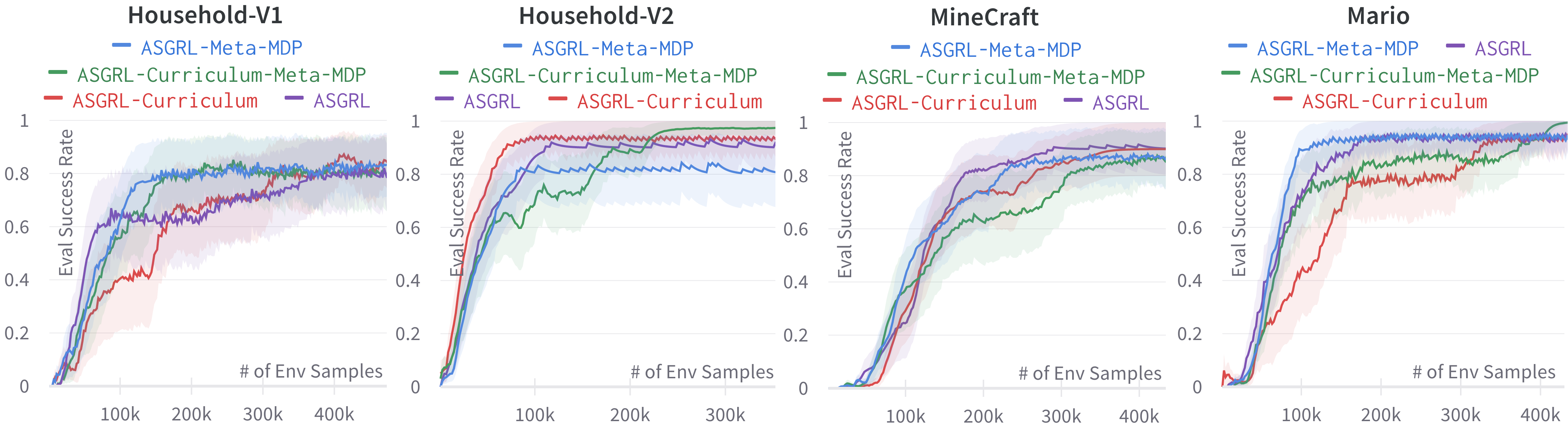}
\caption{Performance when different meta-state representations are used.}
\label{fig:result-diff-meta}
\end{figure}

\section{Implementation Details and Hyper-parameters}
\label{appendix:implementation-details}

For all discrete domains, we use a compact grid encoding. The state is represented as a multi-dimensional vector in which the index of each element corresponds to a specific location and the value of each element corresponds to the type of object that appears at that location. The Household environment is built on the codebase developed by ~\citet{gym_minigrid}.

To balance exploration and exploitation, we use $\epsilon$-greedy in our approach and other baselines. Each skill policy maintains its own $\epsilon_1$ value. $\epsilon_1$ is annealed from $1.0$ to $0.05$ by a factor of $0.95$ whenever the skill successfully reaches a skill terminal state. The meta-controller starts with $\epsilon_2=1.0$ and decreases it by a factor of $0.9$ whenever the low-level skills reach the final goal state(s) until $\epsilon_2=0.05$. Under our curriculum learning setting, the number of diverse skills to learn is increased whenever the $\epsilon_1$ of any skill drops below $0.3$ until the maximum number of skills is reached.

Similar to \cite{yang2018peorl}, the learning rate of each skill policy is also annealed from $1.0$ to $0.1$ by a factor of $0.95$ every time the skill reaches a landmark state. To accelerate the learning process, inspired by the self-imitation learning \cite{oh2018self}, we use a separate replay buffer to store recent ``successful" trajectories. At each parameter update step, additional training data are sampled uniformly from this buffer such that potentially important experience can be used for parameter update more frequently. Also note that, learning from a sparse binary $\mathcal{R}_{meta}$ can be challenging. In practice, we augment it with a shaping reward that assigns $+1$ whenever a target subgoal is reached. This shaping reward biases the meta-controller to select the sequence of skills that can satisfy as many subgoals as possible.

The baselines Plan-HRL and Landmark-HRL use a similar implementation of the low-level RL agent and meta-controller in our approach. But they differ in how the subgoals and the low-level rewards are defined. In Landmark-Shaping and Goal-Q-Learning, only one universal policy is learned. As there is no subgoal being used in Landmark-Shaping and Goal-Q-Learning, we perform $\epsilon$ annealing at the end of each episode regardless of whether the final goal is reached or not. To ensure sufficient exploration, we use a smaller annealing factor $0.995$ for Landmark-Shaping and Goal-Q-Learning.

\clearpage
\section{Symbolic Models}
\label{appendix:model}
\subsection{Household-V1}
\label{appendix:model-dock-v1}
\begin{verbatim}
# Domain Model
(define (domain grid_world)
    (:requirements :strips :typing)
    (:types key - object)
    (:predicates  (has-key)
                  (at-starting-room)
                  (holding ?x - key)
                  (door-open)
                  (door-ajar)
                  (charge)
                  (at-final-room)
                  (at-destination))
    (:action pickup_key
            :parameters (?k - key)
            :precondition (and )
            :effect (and (has-key) (holding ?k)))
    (:action charge_door
            :parameters ()
            :precondition (has-key)
            :effect (and (charged)))
    (:action open_door
            :parameters ()
            :precondition (has-key)
            :effect (and (door-open)))
    (:action pass_through_door
            :parameters ()
            :precondition (and (door-open) (charged))
            :effect (and (at-final-room) (door-ajar)))
    (:action go_to_destination
        :parameters ()
        :precondition (and (at-final-room))
        :effect (and (at-destination)))
)
# Problem Model
(define (problem prob)
    (:domain grid_world)
    (:objects 
        yellow green red - key)
    (:init 
        (at-starting-room))
    (:goal
        (and (at-destination))
))
\end{verbatim}

The landmarks and the immediate ordering is as follows. $\texttt{(has-key)}\prec\texttt{( door-open)}$,$\texttt{(has-key)}\prec\texttt{(charged)}$, $\texttt{(door-open)}\prec\texttt{(at-final-room)}$, $\texttt{(charged)}\prec\texttt{(at-final-room)}$, $\texttt{(at-final-room)}\prec\texttt{(at-destination)}$

\clearpage
\subsection{Household-V2}
\label{appendix:model-dock-v2}
\begin{verbatim}
# Domain Model
(define (domain grid_world)
    (:requirements :strips :typing)
    (:predicates (at-starting-room)
                 (door-open)
                 (door-ajar)
                 (at-final-room)
                 (at-destination))
    (:action pass_through_door
            :parameters ()
            :precondition (and )
            :effect (and (at-final-room)
                         (door-ajar)))
    (:action go_to_destination
        :parameters ()
        :precondition (and (at-final-room))
        :effect (and (at-destination)))
)
# Problem Model
(define (problem prob)
    (:domain grid_world)
    (:objects 
    )
    (:init 
        (at-starting-room))
    (:goal
        (and (at-destination))
))
\end{verbatim}
The landmarks and the immediate ordering is as follows $ \texttt{(at-final-room)}\prec\texttt{(at-destination)}$
\clearpage
\subsection{MineCraft}
\label{appendix:model-minecraft}
\begin{verbatim}
# Domain Model
(define (domain minecraft)
    (:requirements :strips :typing)
    (:types wood - object)
    (:predicates (wood-processed)
                 (at-starting-location)
                 (plank_made)
                 (stick_made)
                 (ladder_made))
    (:action get_processed_wood
            :parameters ()
            :precondition (and )
            :effect (and (wood-processed)))
    (:action make_plank
            :parameters ()
            :precondition (and (wood-processed))
            :effect (and (plank_made)))
    (:action make_stick
            :parameters ()
            :precondition (and (wood-processed))
            :effect (and (stick_made)))
    (:action make_ladder
            :parameters ()
            :precondition (and (stick_made) (plank_made))
            :effect (and (ladder_made)))
)
# Problem Model
(define (problem prob)
    (:domain minecraft)
    (:objects 
        wood0 wood1 - wood)
    (:init 
        (at-starting-location))
    (:goal
        (and (ladder_made))
))
\end{verbatim}
The landmarks and the immediate ordering is as follows $\texttt{(wood-processed)}\prec\texttt{(plank\_made)}$, 
$\texttt{(wood-processed)}\prec \texttt{(stick\_made)}$, $\texttt{(plank\_made)}\prec\texttt{(at-destination)}$ and $\texttt{(stick\_made)}\prec\texttt{(at-destination)}$

\clearpage
\subsection{Mario}
\label{appendix:model-mario}
\begin{verbatim}
# Domain Model
(define (domain Mario)
    (:requirements :strips :typing)
    (:types key - object)
    (:predicates  (has-key)
                  (at-upper-platform)
                  (at-bottom)
                  (at-upper-platform-with-key)
                  (door-open))
    (:action go_down_the_ladder
            :parameters ()
            :precondition (and (at-upper-platform))
            :effect (and (at-bottom) ))
    (:action pickup_key
            :parameters ()
            :precondition (and (at-bottom))
            :effect (and (has-key) ))
    (:action go_up_the_ladder
            :parameters ()
            :precondition (and (has-key) (at-bottom))
            :effect (and (at-upper-platform-with-key)))
    (:action unlock_door
            :parameters ()
            :precondition (and (at-upper-platform-with-key))
            :effect (and (door-open)))
)
# Problem Model
(define (problem prob)
    (:domain Mario)
    (:objects 
    )
    (:init 
        (at-upper-platform))
    (:goal
        (and (door-open))
))
\end{verbatim}
The landmarks and the ordering information for the model is as follows $\texttt{at-upper-platform} \prec \texttt{at-bottom} \prec \texttt{has-key} \prec \texttt{at-upper-platform-with-key} \prec \texttt{door-open}$.